\documentclass[letterpaper]{article} 
\usepackage{aaai25}  
\usepackage{times}  
\usepackage{helvet}  
\usepackage{courier}  
\usepackage[hyphens]{url}  
\usepackage{graphicx} 
\usepackage{natbib}  
\usepackage{caption} 
\usepackage{algorithm}
\usepackage{algorithmic}

\usepackage{newfloat}
\usepackage{listings}
\DeclareCaptionStyle{ruled}{labelfont=normalfont,labelsep=colon,strut=off} 
\floatstyle{ruled}
\newfloat{listing}{tb}{lst}{}
\floatname{listing}{Listing}
\nocopyright 

\title{Macformer: Transformer with Random Maclaurin Feature Attention}

\author{
    Yuhan Guo,
    Lizhong Ding,
    Ye Yuan,
    Guoren Wang
}
\affiliations{
    \textsuperscript{\rm 1}Beijing Institute of Technology\\

    No 5 Zhongguancun South Street, Haidian District\\
    Beijing, China\\
    3220231188@bit.edu.cn
}

\usepackage{bibentry}

\usepackage{multirow}
\usepackage{amsthm}
\usepackage{mathrsfs}
\usepackage{amsmath}
\usepackage{subfig}
\usepackage{amssymb}
\usepackage{booktabs}
\usepackage{bm}
\newtheorem{theorem}{Theorem}
\allowdisplaybreaks[4]

\begin{document}

\maketitle

\begin{abstract}
Random feature attention (RFA) adopts random fourier feature (RFF) methods to approximate the softmax function, resulting in a linear time and space attention mechanism that enables the construction of an efficient Transformer. Inspired by RFA, we propose Macformer, a Transformer architecture that employs random Maclaurin features (RMF) to approximate various dot-product kernels, thereby accelerating attention computations for long sequence. Macformer consists of Random Maclaurin Feature Attention (RMFA) and pre-post Scaling Batch Normalization (ppSBN), the former is an unbiased approximation for dot-product kernelized attention and the later is a two-stage regularization mechanism guaranteeing the error of RMFA. We conducted toy experiments to demonstrate the efficiency of RMFA and ppSBN, and experiments on long range arena (LRA) benchmark to validate the acceleration and accuracy of Macformer with different dot-product kernels. Experiment results of Macformer are consistent with our theoretical analysis.
\end{abstract}

\section{Introduction}

Recently, the Transformer model architecture \cite{transformer} has demonstrated significant success across various deep learning tasks \cite{Asurvey}, such as computer vision \cite{cv1,cv2,cv3}, natural language processing \cite{nlp1,nlp2,nlp3}, multi-modality \cite{mm1,mm2,mm3} and signal processing \cite{sp1,sp2}. The self-attention mechanism \cite{attention} for modeling sequence correlations is a key component of the Transformer, with a time complexity of \( O(n^2) \), which leads to unacceptable overhead when processing long sequence data.

Random feature methods \cite{randomFeature1,randomFeature2} embed the nonlinear feature space (i.e., the Reproducing Kernel Hilbert Space) into a low-dimensional Euclidean space, allowing approximating the kernel function by inner product of random feature mappings, at the cost of introducing an error related to the dimensionality of the random feature space. For some kernel functions with high computational complexity, using random feature map can significantly reduce their computational cost \cite{randomSpeed1,randomSpeed2}. There has been a significant amount of work on improving Softmax attention by approximating shift-invariant kernels using random fourier feature (RFF) \cite{performer,skyformer}. A common approach involves extracting the Gaussian kernel in the exponential operation and approximating it with RFF to linearize Softmax attention \cite{rfa}. We observe that using random Maclaurin features (RMF) \cite{rmf} to approximate dot product kernels can directly linearize Softmax attention, and various dot product kernels can be selected to measure sequence correlations, with softmax being a special case when choosing exponential function.

We propose the Macformer architecture, which introduces the novel Random Maclaurin Feature Attention (RMFA) and pre-post Scaling Batch Normalization (ppSBN) mechanisms in the attention layer. RMFA offers superior computational efficiency for longer sequences, with ppSBN providing its guarantee. Macformer allows us to explore attention improvement methods with RMF from the perspective of dot-product kernels. Our experiments first evaluate the performance of RMFA and ppSBN within the Macformer architecture, demonstrating their effectiveness. Then we validated Macformer on various real-world long sequence datasets provided by the long range arena (LRA) benchmark \cite{lra}. By comparing the performance of Macformer with different dot-product kernels and other models, we demonstrated competitive results, proving different effectiveness for various kernels of the Macformer architecture.

\section{Related Work}

One common motivation shared across several studies is the need to scale transformers to process long sequences efficiently. To address this limitation, various approaches have been proposed. 

Softmax-Free Linear Transformers \cite{rel1} aim to approximate self-attention at linear complexity by replacing the dot-product similarity with a Gaussian kernel function. This innovation allows for a full self-attention matrix to be approximated under low-rank matrix decomposition. Linformer \cite{linformer} approximates the self-attention mechanism with a low-rank matrix, reducing its complexity from $O(n^2)$ to $O(n)$. Methods such as Swin Transformer \cite{rel2} have explored shifted window-based self-attention to reduce computational costs. However, while Swin Transformer addresses some of the challenges associated with traditional transformers, it may still struggle when training on small datasets. ESwin Transformer \cite{rel3} further improves upon this by redesigning the modules of Swin Transformer and introducing simple convolutional components, enhancing the model's performance on small datasets.

Sparse attention patterns have been another area of focus for optimizing transformer models \cite{rel8,rel9} by limiting the reception field of attention computation. Efficient Transformer models in domains such as speech recognition \cite{rel4}, audio tagging \cite{rel5}, and remote sensing image segmentation \cite{rel6} often leverage techniques like knowledge distillation, quantization, and sparse attention \cite{rel7} patterns. Informer \cite{informer} introduces the ProbSparse self-attention mechanism and self-attention distillation techniques, reducing time complexity and memory usage to $O(n\log{n})$.

Random feature attention (RFA) \cite{rfa} proposes reducing computational costs by decomposing Softmax into linear operations. This method uses random feature map \cite{randomFeature1} to approximate the attention. A dissection of Transformer \cite{dissection} introduces a new computational approach by viewing Transformer’s attention mechanism as a kernel smoother, replacing Softmax with the product of symmetric kernels. Performer \cite{performer} introduces the FAVOR+ method to reduce the computational complexity of full-rank Softmax attention to linear, providing theoretical guarantees of low estimation variance and unbiased estimation. Skyformer \cite{skyformer} replaces Softmax with a Gaussian kernel and applies the Nyström method to accelerate computation. Spectraformer \cite{spectraformer} tries to present a unified framework for approximating and learning the shift invariant kernel function in linearized attention.

\section{Preliminaries}

In this section, we introduce Softmax and kernelized attention, as well as RMF. In this paper, the input to the attention mechanism, namely the query, key, and value matrices, is denoted as $\bm{Q},\bm{K},\bm{V}\in \mathbb{R}^{n\times d}$.

\subsection{Softmax and Kernelized Attention}

The attention mechanism \cite{attention} is the key component of the Transformer architecture distinguishing it from other deep-learning model architectures, and it is pivotal to its success. Currently, various variants of attention have emerged in practical applications \cite{attentionSurvey}. This paper focuses on improving traditional Softmax attention \cite{transformer} which is the most widely used and fundamental version.

\textbf{Definition 1} (Softmax Attention) The Softmax attention first measured the similarity between \( \bm{Q} \) and \( \bm{K} \), constructing a convex combination of \( \bm{V} \) as output of the attention layer:

\begin{equation}
{\rm attn}_{\rm Softmax}(\bm{Q},\bm{K},\bm{V})={\rm Softmax}\left ( \frac{\bm{Q}\bm{K}^\top}{\sqrt{d} } \odot \bm{M} \right ) \bm{V},
\nonumber
\end{equation}

\noindent \textit{where $\bm{M}\in \mathbb{R}^{n\times n}$ is an optional mask matrix and $\odot$ means maskfill operation,}

\begin{equation}
{\rm Softmax}(x)=\frac{{\rm exp}(x)}{ {\textstyle \sum_{i}{\rm exp}(x_i)} }.
\nonumber
\end{equation}

It is obvious that the computation of Softmax attention has a time complexity of \( O(n^2) \) for each prediction. When modeling long sequence, the attention layer incurs significant time overhead, becoming a performance bottleneck of the model.

Notably, in the softmax function, the exponential operation serves as the \textbf{similarity function} between the $\bm{Q}$ and $\bm{K}$ sequences. However, the most suitable similarity function can vary across different task scenarios, so it would be beneficial to generalize the exponential operation into a kernel function \( \mathcal{K} \).

\begin{figure*}[htb]
\centering
\includegraphics[width=1\textwidth]{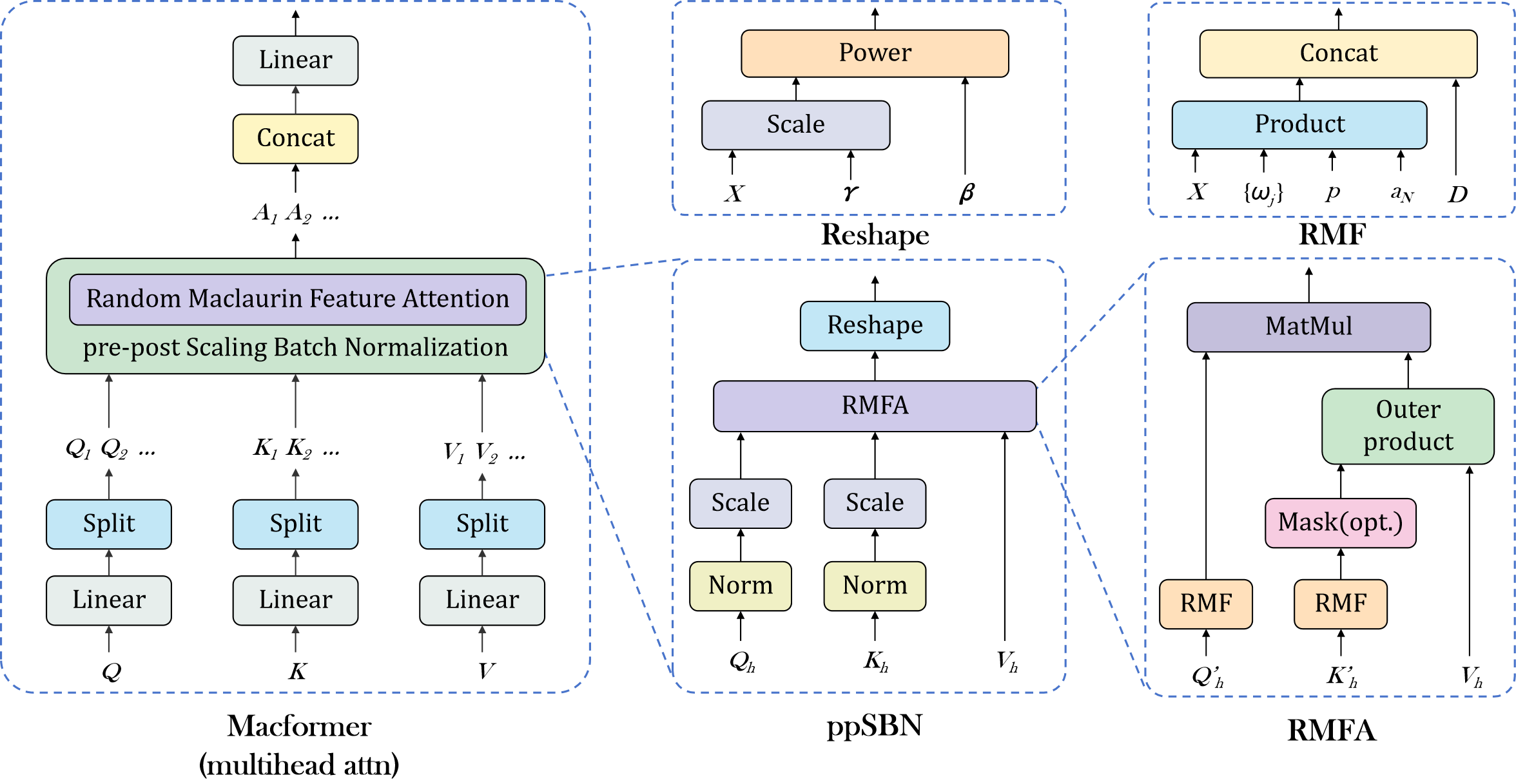}
\caption{The Macformer architecture improves the multi-head attention component of the original Transformer, with RMFA being wrapped by the preSBN and postSBN layers.}
\label{macformer}
\end{figure*}

\textbf{Definition 2} (Kernelized Attention) Given a kernel function \( \mathcal{K}(x, y) \), kernelized attention is defined as the linear combination of \( \bm{V} \) weighted by the normalized \( \mathcal{K}(\bm{Q}, \bm{K}) \).

\begin{equation}
{\rm attn}_{K}(\bm{Q},\bm{K},\bm{V})={ \sum_{i=1}^{n}\frac{\mathcal{K}(\bm{Q} \bm{K}_i^\top/\sqrt{d}) \bm{V} }{ \sum_{j=1}^{n}\mathcal{K}(\bm{Q} \bm{K}_j^\top/\sqrt{d}) } }.
\nonumber
\end{equation}

When \(\mathcal{K}(x,y)=\exp(x\cdot y)\), kernelized attention is equivalent to Softmax attention (we will provide the masked form in the next section). Since the value of the kernel function can be negative, the attention in this case may not be a convex combination of the values.

\subsection{Random Maclaurin Feature}

Previous studies have primarily taken \( \mathcal{K}(x, y) \) as a shift-invariant kernel $\mathcal{K}(x\!-\!y)$ represented by Gaussian kernels and used RFF for approximation. However, in Softmax attention, \( \mathcal{K}(x, y) \) is inherently a dot-product kernel $\mathcal{K}(x\cdot y)$, which can be directly approximated by RMF \cite{rmf}.

\textbf{Definition 3} (RMF)  The Random Maclaurin Feature projection \( \Phi(\cdot) \) maps the input of a dot-product kernel with non-negative Maclaurin coefficients to a random feature vector.

\textit{Let $\Phi (\cdot ):\mathbb{R}^d \to \mathbb{R}^D$ be a nonlinear transformation:}

\begin{equation}
\label{Phi}
\Phi (x)=\sqrt{1/D} \left [ \phi_1 (x),\dots,\phi_D (x) \right ],
\nonumber
\end{equation}

\noindent \textit{where $\phi_i(\cdot):\mathbb{R}^d \to \mathbb{R}$ is a non-linear mapping independent of $i$:}

\begin{equation}
\label{phi}
\phi_i (x)=\sqrt{a_N p^{N+1}}  {\textstyle \prod_{j=1}^{N} \left < \omega_j,x\right > } ,
\nonumber
\end{equation}

\noindent \textit{where $p>1$ is a hyperparameter, \( N \) is sampled from \( \mathbb{N} \) with $\mathbb{P}[N=\eta]=\frac{1}{p^{\eta+1}} $, \( a_N \) is the \( N \)th coefficient of the Maclaurin expansion of the kernel function \( \mathcal{K}(\cdot) \), and \( \omega_j \) is a \( d \)-dimensional Rademacher vector independent of \( j \). }

\textit{Then we can approximate kernel with random features:}

\begin{equation}
\Phi (x)\cdot \Phi (y)\approx{\mathcal{K}(x\cdot y)}.
\nonumber
\end{equation}

The $\Phi (\cdot)$ defined here can also perform random projection on high-dimensional tensors, acting on the last dimension. Under specific circumstances, the dot product of the mapped random features is theoretically an unbiased estimate of the original dot product kernel $\mathcal{K}(x,y)=\mathcal{K}(x\cdot y)$. RMF directly utilize random features to unbiasedly approximate any dot product kernel function with only non-negative Maclaurin coefficient \cite[Lemma 7]{rmf}. In the Macformer architecture, our core work is to approximate dot-product kernelized attention with RMF, modifying the computational graph thus achieving better performance on long-sequence scenario.

\section{Macformer}

In this section, we introduce the Macformer architecture, as shown in Figure \ref{macformer}, an improved Transformer model with a better performance on long sequence scenarios. 

Similar to Transformer, Macformer accepts a sequence of data, processing it through layers with multihead attention. But Macformer applies the ppSBN mechanism before and after RMFA. We then demonstrate functions and theoretical analysis for its two main components.

\begin{figure*}[htb]
\centering
  \subfloat[Softmax Attention]
  {
      \label{attnCal}\includegraphics[width=0.5\textwidth]{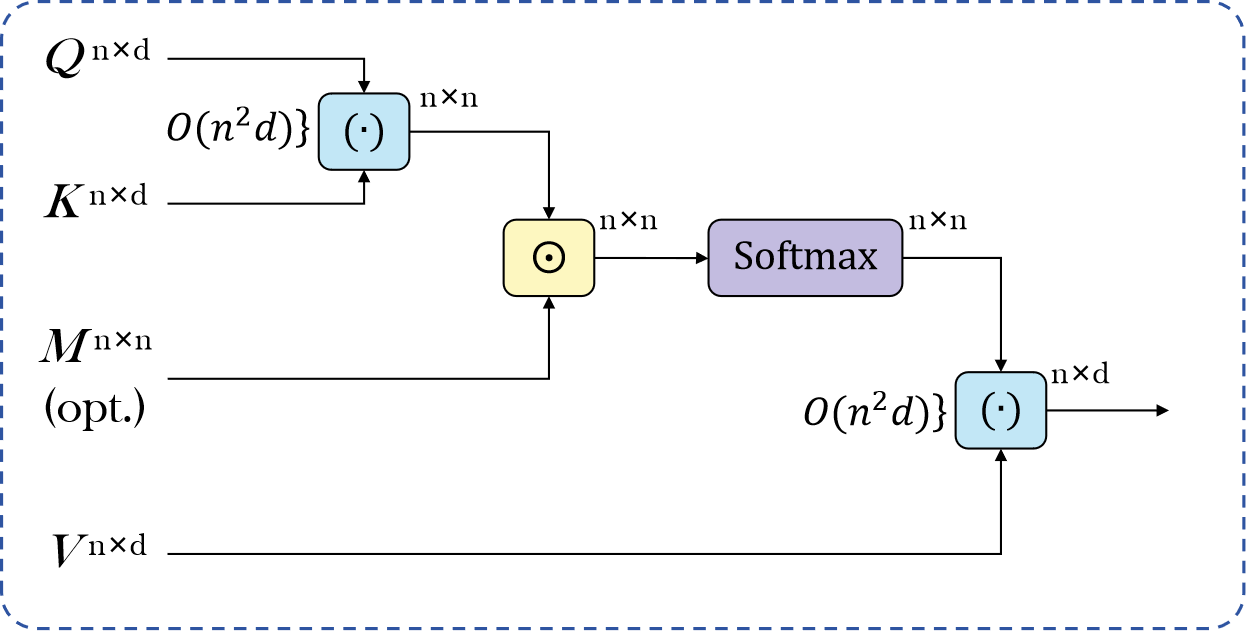}
  }
  \subfloat[RMFA]
  {
      \label{rmfaCal}\includegraphics[width=0.5\textwidth]{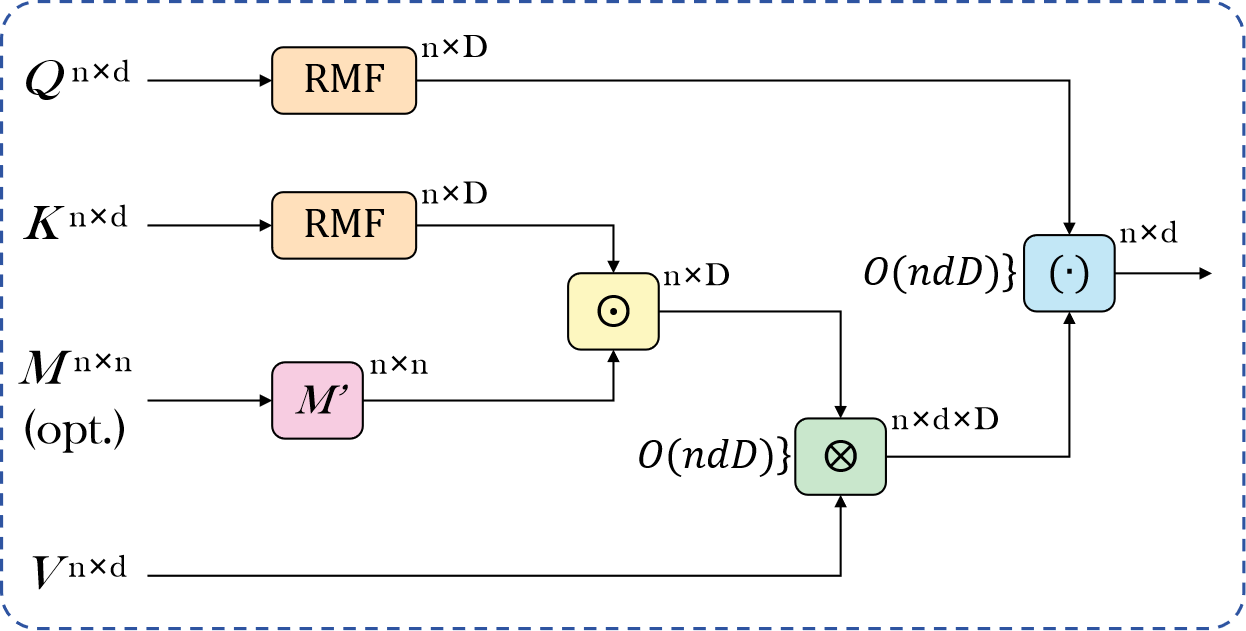}
  }
\caption{Computation graphs for Softmax attention and RMFA. In each figure, the data on the left represents the input to the attention layer. Here, operators $(\cdot)$, $\odot$, and $\otimes$ respectively denote matrix multiplication, mask fill, and outer product. The main time complexity caused by computations is marked on the left side of the operators, and the dimensions of data and intermediate results are indicated with superscripts.}
\label{calPro}
\end{figure*}

\subsection{Random Maclaurin Feature Attention}

We begin with the fact that any dot-product kernel with non-negative Maclaurin coefficients can be approximated by RMF, the same applies to kernelized attention.

\begin{equation}
\mathcal{K}(\bm{Q} \bm{K}^\top/\sqrt d)\approx{\Phi(\bm{Q}/d^{\frac{1}{4}}) \Phi^\top(\bm{K}/d^{\frac{1}{4}})}.
\nonumber
\end{equation}

Similar to RFA \cite{rfa}, which improves the attention computation path with Gaussian kernels, we approximate the dot-product kernel with RMF, allowing for the linear time approximation of attention.

When approximation kernelized attention, we have

\begin{align*}
\label{rmfa}
\begin{split}
&\quad {\rm attn}_\mathcal{K}(\bm{Q},\bm{K},\bm{V}) \\
&={ \sum_{i=1}^{n}\frac{\mathcal{K}(\bm{Q} \bm{K}_i^\top/\sqrt{d}\odot \bm{M})\bm{V} }{ \sum_{j=1}^{n}\mathcal{K}(\bm{Q} \bm{K}_j^\top/\sqrt{d}\odot \bm{M}) } }\\
&={ \sum_{i=1}^{n}\frac{[\mathcal{K}(\bm{Q} \bm{K}_i^\top/\sqrt{d})\odot \bm{M}'] \bm{V} }{ \sum_{j=1}^{n}\mathcal{K}(\bm{Q} \bm{K}_j^\top/\sqrt{d})\odot \bm{M}' } }\\
&\approx{{ \sum_{i=1}^{n}\frac{[\Phi(\bm{Q}/d^{\frac{1}{4}}) \Phi^\top(\bm{K}_i/d^{\frac{1}{4}})\odot \bm{M}'] \bm{V} }{ \sum_{j=1}^{n}\Phi(\bm{Q}/d^{\frac{1}{4}}) \Phi^\top(\bm{K}_j/d^{\frac{1}{4}})\odot \bm{M}' } }}\\
&=\frac{ \Phi (\bm{Q}/d^{\frac{1}{4}})  {\textstyle \sum_{i=1}^{n}\left [ \Phi^\top (\bm{K}_i/d^{\frac{1}{4}})\odot \bm{M}' \right ]\otimes \bm{V}}}{ \Phi (\bm{Q}/d^{\frac{1}{4}}) \sum_{j=1}^{n} \left [ \Phi ^\top(\bm{K}_j/d^{\frac{1}{4}})\odot \bm{M}' \right ] }\\
&={\rm RMFA}_\mathcal{K}(\bm{Q},\bm{K},\bm{V}),
\end{split}
\end{align*}

\noindent \textit{where $\otimes$ represents the outer product of vectors. $\bm{M}$ is converted to \( \bm{M}'\in \mathbb{R}^{n\times n} \) as follows, making specified positions of the attention matrix invalid:}

\begin{equation}
\bm{M}'_{i,j}=\begin{cases}
1,\bm{M}_{i,j}\ge 0
 \\0,\bm{M}_{i,j}<0
\end{cases}.
\nonumber
\end{equation}

RMFA serves as a drop-in replacement of Softmax attention, directly integrating into existing transformer architectures without significant modifications. By approximating the Softmax operation with RMF, RMFA avoids the direct multiplication of large matrices \(\bm{Q}\) and \(\bm{K}\), thereby saving computational costs. 

We can analyze the time complexity of RMFA compared to Softmax attention from Figure \ref{calPro}. 

In Figure \ref{attnCal}, the direct multiplication of large matrices \(\bm{Q}\) and \(\bm{K}\) results in a time complexity of \(O(n^2d)\). After passing through the Softmax function, another large matrix multiplication costing \(O(n^2d)\) is required to obtain the attention output. 

On the contrary, as shown in Figure \ref{rmfaCal}, RMFA avoids large matrix multiplications by decomposing the dot-product kernel into a product of random features that can be factored out, results in a time complexity of \(O(ndD)\). In long sequence tasks where \( n \gg D \), RMFA can significantly improve computational efficiency.

Due to the introduction of randomness by the random feature method, theoretical guarantees are required for the approximation effectiveness of RMFA. 

Let $\ell_\rho (0,\lambda )$ denotes the set $\{ {\left \| \bm{X} \right \|}_\rho \le \lambda \}$. According to Schoenberg (1942), Theorem 2, when $\bm{Q},\bm{K}\in \ell_2 (0,1)$, RFMA serves as an unbiased estimator for the dot-product kernelized attention.

\begin{theorem}
\label{expTheo}
\textit{Suppose attention inputs $\bm{Q},\bm{K}\in \ell_2 (0,1)$, $ \Phi (\cdot)$ defines a Random Maclaurin Feature map for a dot-product kernel $\mathcal{K}(\cdot)$, then for every $\bm{V}\subset \mathbb{R}^{n \times d}$, we have $\mathbb{E}[{\rm RMFA}_\mathcal{K}(\bm{Q},\bm{K},\bm{V})]={\rm attn}_\mathcal{K}(\bm{Q},\bm{K},\bm{V})$.}
\end{theorem}

\begin{proof}
We have
\begin{align*}
\nonumber
&\quad \mathbb{E}[{\rm RMFA}_\mathcal{K}(\bm{Q},\bm{K},\bm{V})] \\
&=\sum_{i=1}^{n}\frac{\mathbb{E} \left [ \frac{1}{D} {\textstyle \sum\phi(\bm{Q}/d^{\frac{1}{4}})\phi^\top(\bm{K}_i/d^{\frac{1}{4}})} \odot \bm{M}' \right ] \bm{V} }{ \sum_{j=1}^{n}\mathbb{E} \left [ \frac{1}{D} {\textstyle \sum\phi(\bm{Q}/d^{\frac{1}{4}})\phi^\top(\bm{K}_j/d^{\frac{1}{4}})} \odot \bm{M}' \right ] }\\
&=\sum_{i=1}^{n}\frac{\!\!\mathbb{E}_N\!\! \left [{a_N p^{N+1} \mathbb{E}_\omega \left [ \prod \left \langle \omega ,\bm{Q} \right \rangle \left \langle \omega ,\bm{K}_i \right \rangle \right /\sqrt d]} \odot \bm{M}' \right ] \bm{V} }{ \sum_{j=1}^{n}\!\!\mathbb{E}_N\!\! \left [ {a_N p^{N+1} \mathbb{E}_\omega \left [ \prod \left \langle \omega ,\bm{Q}\right \rangle \left \langle \omega ,\bm{K}_j\right \rangle  \right /\!\!\sqrt d] }\! \odot \bm{M}'\! \right ] }\\
&=\sum_{i=1}^{n}\frac{\mathbb{E}_N \left [{a_N p^{N+1} {(\bm{Q} \bm{K}_i^\top/\sqrt d)^N}} \odot \bm{M}' \right ] \bm{V} }{ \sum_{j=1}^{n}\mathbb{E}_N \left [ {a_N p^{N+1} (\bm{Q} \bm{K}_j^\top/\sqrt d)^N} \odot \bm{M}' \right ] }\\
&=\sum_{i=1}^{n}\frac{ {\textstyle \sum_{\eta =0}^{\infty}a_\eta \frac{1}{p^{\eta+1}}p^{\eta+1}(\bm{Q} \bm{K}_i^\top/{\sqrt d})^\eta } \odot \bm{M}' \bm{V} }{ \sum_{j=1}^{n} {\left [ \textstyle \sum_{\eta =0}^{\infty}a_\eta \frac{1}{p^{\eta+1}}p^{\eta+1}(\bm{Q} \bm{K}_j^\top/{\sqrt d})^\eta \right ] } \odot \bm{M}' } \\
&={ \sum_{i=1}^{n}\frac{\mathcal{K}(\bm{Q} \bm{K}_i^\top/\sqrt{d}\odot \bm{M}) \bm{V} }{ \sum_{j=1}^{n}\mathcal{K}(\bm{Q} \bm{K}_j^\top/\sqrt{d}\odot \bm{M}) } }={\rm attn}_\mathcal{K}(\bm{Q},\bm{K},\bm{V}).
\end{align*}
\end{proof}

Theorem \ref{expTheo} indicates that as the number of computations approaches infinity, the mean of RMFA equals kernelized attention. However, considering efficiency concerns, we only compute RMFA once every iteration, which poses a requirement for the stability of the approximation effectiveness of RMFA.

Thus, we provide a theoretical guarantee for the approximation error bound of RMFA when the scale of \( \bm{V} \) is limited.

\begin{theorem}
\label{pacTheo}
\textit{Suppose attention inputs $\bm{Q}\in \ell_2 (0,1)$, $\bm{K}\in \ell_2 (0,1)$ and $\left | \bm{V}_{ij} \right | \le L$, for any $\varepsilon >0$, we have}
\end{theorem}

\begin{equation}
\label{pac}
\begin{split}
&\quad \mathbb{P}\left ( \left | {\rm RMFA}_\mathcal{K}(\bm{Q},\bm{K},\bm{V}) -{\rm attn}_\mathcal{K}(\bm{Q},\bm{K},\bm{V}) \right | > \epsilon \right ) \\
&\le 2D\ {\rm exp}\left (- \frac{D\epsilon ^2}{2L^2}  \right ) .
\end{split}
\nonumber
\end{equation}

\begin{proof}
We begin with
\begin{equation}
\forall \bm{Q},\bm{K}\in  \ell_2 (0,1),\left | \bm{V}_{ij} \right | \le L,
\nonumber
\end{equation}
\begin{equation}
{\rm RMFA}_\mathcal{K}(\bm{Q},\bm{K},\bm{V})_{ij}\in [-L,L],
\nonumber
\end{equation}

then, using Hoeffding's inequality,

\begin{align*}
&\quad \mathbb{P}\left ( \left | {\rm RMFA}_\mathcal{K}(\bm{Q},\bm{K},\bm{V}) -{\rm attn}_\mathcal{K}(\bm{Q},\bm{K},\bm{V}) \right | > \epsilon \right ) \\
&=\mathbb{P}\Bigg ( \exists\phi_t:\Bigg | \sum_{i=1}^{n}\frac{ \left [ {\phi_t(\bm{Q}/d^\frac{1}{4} )\phi_t(\bm{K}_i/d^\frac{1}{4})} \odot \bm{M}' \right ]  \bm{V} }{ \sum_{j=1}^{n} {\phi_t(\bm{Q}/d^\frac{1}{4})\phi_t(\bm{K}_j/d^\frac{1}{4})} \odot \bm{M}'} \\
&\quad-{\rm attn}_\mathcal{K}(\bm{Q},\bm{K},\bm{V}) \Bigg | >\epsilon   \Bigg ) \\
&\le {\textstyle \sum_{t=1}^{D}} \mathbb{P}\Bigg ( \Bigg | \sum_{i=1}^{n}\frac{ \left [ {\phi_t(\bm{Q}/d^\frac{1}{4})\phi_t(\bm{K}_i/d^\frac{1}{4})} \odot \bm{M}' \right ]  \bm{V} }{ \sum_{j=1}^{n} {\phi_t(\bm{Q}/d^\frac{1}{4})\phi_t(\bm{K}_j/d^\frac{1}{4})} \odot \bm{M}'} \\
&\quad-{\rm attn}_\mathcal{K}(\bm{Q},\bm{K},\bm{V}) \Bigg | >\epsilon   \Bigg )\\
&\le D\cdot 2{\rm exp}\left [ -\frac{2D^2\epsilon ^2}{D(2L)^2}  \right ] =2D\ {\rm exp}\left (- \frac{D\epsilon ^2}{2L^2} \right ).
\nonumber
\end{align*}
\end{proof}

Thereby, we have demonstrated that RMFA is equivalent to kernelized attention in expectation and that the approximation performance theoretically improves as $D$ increases or $L$ decreases. In applications, one can choose appropriate parameters  to enable RMFA to serve as a replacement for kernelized attention in the model, and specifically select $\mathcal{K}(\cdot)=\rm{exp}(\cdot)$ for approximating Softmax attention.

\begin{table}[ht]
  \centering
    \begin{tabular}{ccc}
    \toprule
    $\mathcal{K}$     & $f(x\cdot y)$     & $a_N$ \\
    \midrule
    ${\rm exp}$   & ${\rm exp}(x\cdot y)$    & $1/N!$ \\
    ${\rm inv}$ & $1/(1\!-\!x\cdot y)$    & $1$ \\
    ${\rm log}$   & $1\!-\!{\rm log}(1\!-\!x\cdot y)$    & $\frac{1}{{\rm min}(1,N)}$ \\
    ${\rm trigh}$ &  ${\rm sinh}(x\cdot y)+{\rm cosh}(x\cdot y)$   & $1/N!$ \\
    ${\rm sqrt}$  & $2\!-\!\sqrt{1\!-\!x\cdot y}$    & $\frac{{\rm max}(1,2N-3)}{2^N N!}$ \\
    \bottomrule
    \end{tabular}
  \caption{The dot-product kernel functions we use in this paper along with their non-negative Maclaurin coefficients.}
\label{kernels}
\end{table}

In this work, we tested the performance of five dot-product kernels with non-negative Maclaurin coefficients, as shown in Table \ref{kernels}. It is worth noting that the domain of dot-product kernels \( \mathcal{K}_{\rm inv} \), \( \mathcal{K}_{\rm log} \), and \( \mathcal{K}_{\rm sqrt} \) requires limited inputs less than $1$, a condition that will be ensured also by the ppSBN mechanism.

\subsection{Pre-post Scaling Batch Normalization}

RMFA can directly replace Softmax attention in the original Transformer model without requiring other modifications to the architecture. However, due to the uncontrolled scaling of the inputs to the attention layer, the approximation performance may not be guaranteed. According to Theorem \ref{expTheo}, to strictly ensure that RMFA is an unbiased estimate of kernelized attention, the input space needs to be constrained to $\ell_2 (0,1)$[Schoenberg (1942), Theorem 2].

\begin{algorithm}[htb]
\caption{Pre-post Scaling Batch Normalization}
\label{ppSBN}
\textbf{Input}: attention input $\bm{Q}$, $\bm{K}$, $\bm{V}$\\
\textbf{Parameter}: trainable parameter $\beta$, $\gamma$; hyperparameter $\varepsilon$\\
\textbf{Output}: attention output $att$
\begin{algorithmic}[1] 
\STATE $\bm{Q},\bm{K}\gets \frac{\bm{Q}-\bm{\mu_Q}}{\sqrt{\bm{\sigma_Q}+\bm{\varepsilon}}},\frac{\bm{K}-\bm{\mu_K}}{\sqrt{\bm{\sigma_K}+\bm{\varepsilon}}}$
\STATE $\bm{Q},\bm{K}\gets \frac{\bm{Q}}{\left \| \bm{Q} \right \|_2 },\frac{\bm{K}}{\left \| \bm{K} \right \|_2 }$
\STATE $att\gets {\rm RMFA}_\mathcal{K}(\bm{Q},\bm{K},\bm{V})$
\STATE $att\gets {(\gamma \cdot att)}^\beta$
\end{algorithmic}
\end{algorithm}

Inspired by the idea of reshaping data expectations and variances through Batch Normalization \cite{batchNorm}, we propose a two-stage regularization method called ppSBN. 

As shown in Algorithm \ref{ppSBN}, we first normalize and scale $\bm{Q},\bm{K}$ to ensure that $\bm{Q}^{SBN},\bm{K}^{SBN}\in \ell_2 (0,1)$, then feed them into RMFA layer for computation. The mean value matrices $\bm{\mu_Q},\bm{\mu_K}$ and variance matrices $\bm{\sigma_Q},\bm{\sigma_K}$ are unsqueezed from vectors, and $\bm{\varepsilon}$ is the all-$\varepsilon$s matrix. The output of RMFA is rescaled by trainable parameters \( \beta \) and \( \gamma \), serving as the final output of the attention layer.We can mathematically demonstrate that when approximating \( {\rm attn}_{\rm exp} \), the ppSBN mechanism can restore the scale of the attention while ensuring that the input space remains constrained.

\begin{theorem}
\label{ppTheo}
\textit{Suppose $\bm{Q}^{SBN},\bm{K}^{SBN}$ denote $\bm{Q},\bm{K}$ after scaling and batch normalization, we have}
\begin{equation}
\label{ppsbn}
{\rm RMFA}_{\rm exp}(\bm{Q}^{SBN},\bm{K}^{SBN},\bm{V}) \approx{\frac{1}{t} {\rm RMFA}_{\rm exp}(\bm{Q},\bm{K},\bm{V})}^\frac{1}{r} ,
\nonumber
\end{equation}

\noindent where $r=\left \| \bm{Q} \right \|_2 \left \| \bm{K} \right \|_2 \sqrt{(\bm{\sigma_Q}+\bm{\varepsilon})(\bm{\sigma_K}+\bm{\varepsilon})}$ dependents on the data and $t={\rm Softmax}(\bm{\mu_Q} \bm{K}^\top/\sqrt{d})^{\frac{1}{r}}/\bm{V}^{\frac{r-1}{r}}$.
\end{theorem}

\begin{proof}
\begin{equation}
\begin{split}
&\quad {\rm RMFA}_{\rm exp}(\bm{Q}^{SBN}\!\!\!,\bm{K}^{SBN}\!\!\!,\bm{V})\!\!\approx{\!{\rm attn}_{\rm exp}(\bm{Q}^{SBN}\!\!\!,\bm{K}^{SBN}\!\!\!,\bm{V})} \\[6pt]
&={\rm Softmax}\left [ \frac{(\bm{Q}-\bm{\mu_Q})(\bm{K}-\bm{\mu_K})^\top }{\left \| \bm{Q} \right \|_2 \left \| \bm{K} \right \|_2 \sqrt{(\bm{\sigma_Q}+\bm{\varepsilon} )(\bm{\sigma_K}+\bm{\varepsilon})}} \bigg/\sqrt d\right ] \bm{V}\\[6pt]
&={\sum_{i=1}^{n}\frac{{\rm exp}\left [ \left ( {\bm{Q}\bm{K}_i^\top\!-\!\bm{\mu_Q}\bm{K}_i^\top\!-\!\bm{Q}\bm{\mu_K}^\top\!+\!\bm{\mu_Q}\bm{\mu_K}^\top}\right )/{r\sqrt{d}}  \right ]  \bm{V} }{ \sum_{j=1}^{n}\!{\rm exp}\!\left [ \left ( {\bm{Q} \bm{K}_j^\top\!-\!\bm{\mu_Q}\bm{K}_j^\top\!-\!\bm{Q}\bm{\mu_K}^\top\!+\!\bm{\mu_Q}\bm{\mu_K}^\top}\right ) \!\!/{r\sqrt{d}} \right ] } }\\[6pt]
&=\sqrt[r]{ {\sum_{i=1}^{n}\!\frac{{\rm exp}\!\left ( {\bm{Q} \bm{K}_i^\top}/{\sqrt{d}}\right )}{ \sum_{j=1}^{n}\!{\rm exp}\!\left ( \frac{\bm{Q} \bm{K}_j^\top}{\sqrt{d}}\right )}} {\bigg/{\sum_{i=1}^{n}\!\frac{{\rm exp}\!\left ( {\bm{\mu_Q} \bm{K}_i^\top}/{\sqrt{d}}\right )}{ \sum_{j=1}^{n}\!{\rm exp}\!\left ( \frac{\bm{\mu_Q}\bm{K}_j^\top}{\sqrt{d}}\right )}} } \bm{V}^r}\\[6pt]
&=\frac{1}{t}\!\sqrt[r]{{\sum_{i=1}^{n}\!\!\frac{{\rm exp}\!\!\left ( {\,\bm{Q} \bm{K}_i^\top}/{\sqrt{d}}\right )}{ \sum_{j=1}^{n}\!\!{\rm exp}\!\!\left ( \frac{\bm{Q} \bm{K}_j^\top}{\sqrt{d}}\right )}}\bm{V}} \!\!=\!\!\frac{1}{t}\!{\left [ {\rm Softmax}\!\!\left ( \frac{\bm{Q} \bm{K}^\top}{\sqrt d} \right ) \!\!\bm{V}\! \right ] }^\frac{1}{r} \\[6pt]
&=\frac{1}{t} {{\rm attn}_{\rm exp}(\bm{Q},\bm{K},\bm{V})}^\frac{1}{r}\approx{\frac{1}{t} {{\rm RMFA}_{\rm exp}(\bm{Q},\bm{K},\bm{V})}^\frac{1}{r}}.
\end{split}
\nonumber
\end{equation}
\end{proof}

\begin{figure*}[ht]
\centering
  \subfloat[Loss]
  {
      \label{ppsbn:subfig1}\includegraphics[width=0.32\textwidth]{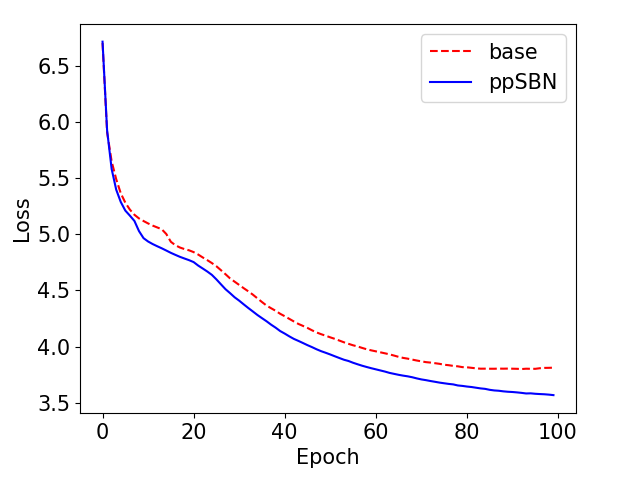}
  }
  \subfloat[PPL]
  {
      \label{ppsbn:subfig2}\includegraphics[width=0.32\textwidth]{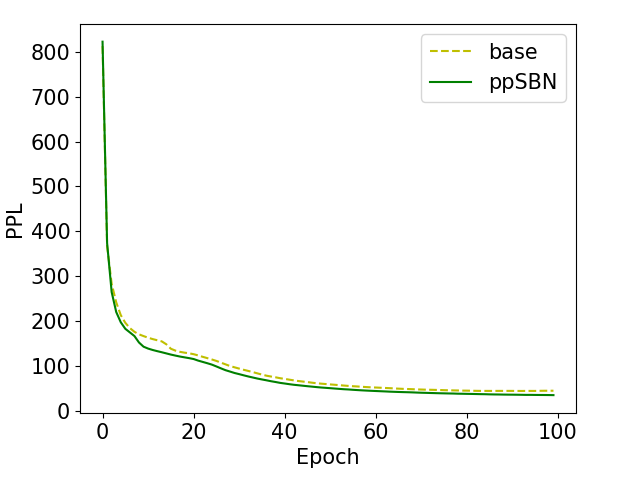}
  }
  \subfloat[BLEU]
  {
      \label{ppsbn:subfig3}\includegraphics[width=0.32\textwidth]{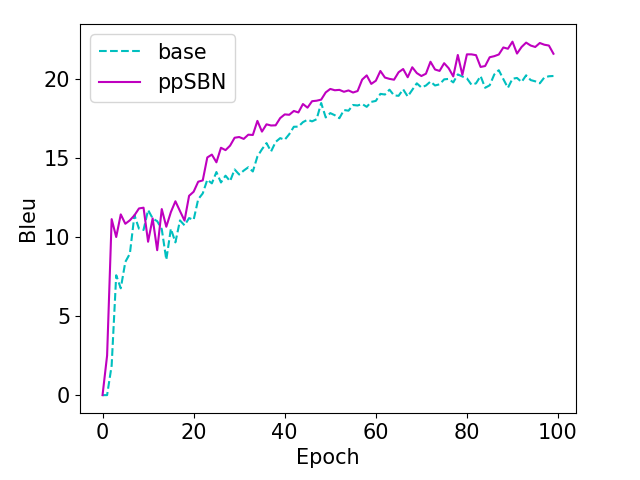}
  }
\caption{The loss, perplexity, and Bleu scores of the traditional Transformer with and without ppSBN across training epochs. In each plot, solid lines represent the Transformer with ppSBN, while dashed lines represent the Transformer without ppSBN.}
\label{ppsbnResult}
\end{figure*}

The trainable parameters \( \gamma \) and \( \beta \) in Algorithm \ref{ppSBN} fit the \( t \) and \( r \) in Theorem \ref{ppTheo}, thus preserving the scale of data. Constraining the scale of the data not only provides better theoretical guarantees for RMFA but also allows us to select kernel functions with a limited domain, as the kernel referred in Table \ref{kernels}. 

Additionally, the regularization effect limits model complexity, theoretically enhancing its generalization further by avoiding over-fit. In next section, we will validate the effectiveness of ppSBN and extend it to any dot-product kernelized attention.

\section{Experiment}

In the experiments, we sequentially validate the effectiveness of RMFA, ppSBN, and Macformer, and their consistency with our theoretical analysis. For all experiments, we set the hyperparameter \( p = 2 \).

\subsection{Evaluation of RMFA}

To evaluate the approximation effectiveness of RMFA for kernelized attention, we conducted simulation experiments using \( \mathcal{K}_{\text{exp}} \) on the generated dataset. We randomly generate data of size $16\times 8$ (simulating a batch size of $16$ and $8$ attention heads), with fixed data dimension \( d = 64 \), and lengths ranging from $200$ to $4000$ uniformly. The dimension of the random feature mapping grows exponentially with a base of $2$. Therefore, the feature mapping here is \( \Phi :(16\times 8\times length\times 64)\to (16\times 8\times length\times D) \).

Each generated data is first preprocessed using pre-SBN (with $\varepsilon =10^{-12}$) and then fed into both Softmax attention and ${\rm RMFA}_{\rm exp}$, yielding their output results and computation time. Each \( (length, D) \) pair undergoes 100 repeated experiments, and the average is taken. We use logarithmized Normalized Mean Squared Error (NMSE) to measure the error of RMFA, and the logarithmized acceleration ratio to measure the relative speed of RMFA to Softmax attention.

\begin{figure}[ht]
\centering
\subfloat[Error]
  {
      \label{layerResult:subfig1}\includegraphics[width=0.5\columnwidth]{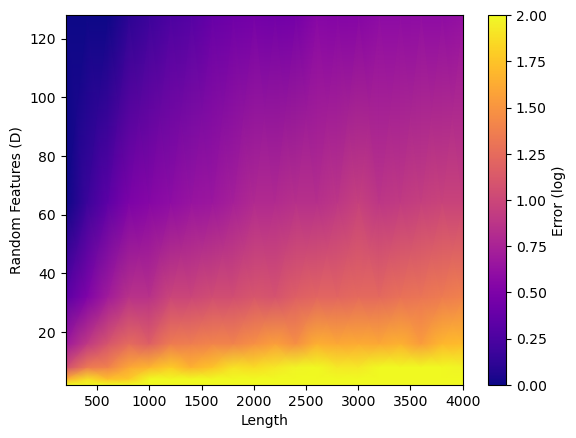}
  }
  \subfloat[Acceleration]
  {
      \label{layerResult:subfig2}\includegraphics[width=0.5\columnwidth]{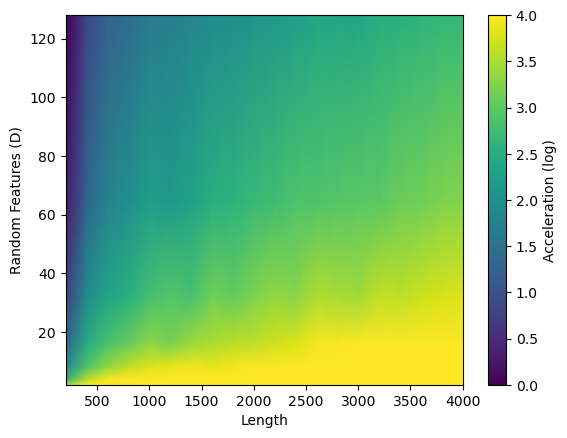}
  }
\caption{The error (a) and acceleration (b) of ${\rm RMFA}_{\rm exp}$ compared to Softmax attention for different sequence lengths and values of \( D \). The data in the figure has been subjected to smoothing, where darker colors represent smaller data values and lighter colors represent larger data values.}
\label{layerResult}
\end{figure}

Figure \ref{layerResult:subfig1} illustrates the error introduced by random features in RMFA. With the length fixed, increasing \( D \) leads to a rapid decrease in error, but this reduction comes at the cost of diminishing marginal returns, consistent with our Theorem \ref{pacTheo}. However, as the length increases, the error also increases. This is because longer data introduces more randomness in the random feature mapping, necessitating an increase in the mapping dimension \( D \) to maintain approximation accuracy.

In Figure \ref{layerResult:subfig2}, the speed of RMFA is superior to Softmax attention in all cases. As the dimension \( D \) increases, the computational cost of RMFA increases, leading to a decrease in the acceleration ratio. However, with \( D \) fixed, as the sequence length increases, the acceleration ratio significantly increases, corroborating our analysis before: RMFA reduces computational overhead for longer sequence lengths.

\begin{table*}[htb]
\renewcommand{\arraystretch}{1.1}
\centering
\resizebox{\textwidth}{!}{
    \begin{tabular}{ccccrcccrccc}
    \toprule
    \multirow{2}[4]{*}{\textbf{Model}} & \multicolumn{3}{c}{\textbf{LRA Text}} &       & \multicolumn{3}{c}{\textbf{LRA Listops}} &       & \multicolumn{3}{c}{\textbf{LRA Retrieval}} \\
\cmidrule{2-4}\cmidrule{6-8}\cmidrule{10-12}          & \textbf{Time} & \textbf{Memory} & \textbf{Accuracy} &       & \textbf{Time} & \textbf{Memory} & \textbf{Accuracy} &       & \textbf{Time} & \textbf{Memory} & \textbf{Accuracy} \\
    \midrule
    Transformer & 1.000  & 1.000  & 63.310  &       & 1.000  & 1.000  & 37.300  &       & 1.000  & 1.000  & 75.040  \\
    \midrule
    ${\rm Transformer}_{\rm RFA}$   & 0.783  & 1.367  & \underline{65.220}  &       & 0.836  & 1.360  & 37.150  &       & 0.900  & 2.389  & \underline{77.840}  \\
    ${\rm Macformer}_{\rm exp}$   & \underline{0.311}  & 1.339  & 64.060  &       & 1.057  & 2.564  & \underline{38.360}  &       & 0.769  & 1.890  & 70.740  \\
    ${\rm Macformer}_{\rm inv}$   & \underline{0.311}  & 1.339  & 63.860  &       & 0.847  & 2.564  & 37.950  &       & 0.823  & 1.890  & 70.840  \\
    ${\rm Macformer}_{\rm trigh}$ & \underline{0.311}  & 1.339  & 64.060  &       & \underline{0.647}  & 2.564  & \underline{38.360}  &       & 0.800  & 1.890  & 74.230  \\
    ${\rm Macformer}_{\rm log}$   & 0.710  & 1.339  & 63.960  &       & 0.751  & 2.564  & 37.550  &       & \underline{0.348}  & 1.890  & 70.730  \\
    ${\rm Macformer}_{\rm sqrt}$  & 0.734  & 1.339  & 63.840  &       & 0.839  & 2.564  & \underline{38.360}  &       & 0.727  & 1.890  & 70.080  \\
    \bottomrule
    \end{tabular}%
}
  \caption{Experimental results on three LRA benchmark tasks. We compared the models of Softmax attention (i.e., Transformer), RFA, and five kinds of Macformer with different dot-product kernels. We report the training time, memory usage and accuracy, with the time and memory data normalized to the base Transformer.}
\label{macformerTable} 
\end{table*}

\subsection{Toy Experiment for ppSBN}

Furthermore, we conducted a toy experiment to assess the impact of ppSBN on entire model. We choose machine translation experiment on Multi30K \cite{multi30k} dataset as the test task, where we used the traditional Transformer as the base model and incorporated the ppSBN mechanism before and after the attention layer for comparison. The loss function chosen is cross-entropy.

The experimental results depicted in Figure \ref{ppsbnResult} demonstrate that ppSBN outperforms the base model across three metrics. Specifically, ppSBN performs comparably to the base model in terms of perplexity (\ref{ppsbn:subfig2}), exhibiting only a slight advantage. However, in terms of loss (\ref{ppsbn:subfig1}) and Bleu (\ref{ppsbn:subfig3}) scores, ppSBN exhibits a notable advantage over the base model. This indicates that ppSBN serving as a method to constrain the input space to ensure the approximation of RMFA, also enhances model training effectiveness due to its parameter regularization effect.

\subsection{Macformer on LRA benchmark}

In our third experiment, we evaluated the performance of Macformer on the LRA benchmark \cite{lra}. The LRA is a standard test used to assess and compare different Transformer models' ability to handle long sequence data. We organized our experiments around three tasks from this benchmark: byte-level text classification, long listops, and byte-level document retrieval.

\begin{itemize}
\item Text: The byte-level text classification task uses the character-level IMDb movie review dataset to evaluate a model's ability to handle long documents. This task requires the model to address compositionality issues, where characters form words and then higher-level phrases, ultimately completing a binary text classification task.
\item Listops: The long listops task assesses a model's ability to handle hierarchically structured data. The input is a string representing a series of operations on single-digit integers, including MAX, MEAN, MEDIAN, and nested structures. The model needs to output the computed result, which is also a single-digit integer.
\item Retrieval: The byte-level document retrieval task requires the model to determine whether there is a citation relationship between two documents. The task involves compressing the representations of two documents and calculating their similarity score through a linear classifier, resulting in a binary classification output.
\end{itemize}

We conducted all experiments on one NVIDIA RTX A6000 48G, working on a baseline model implemented by \citet{skyformer} with PyTorch. The Macformer model we used shares the same parameter settings with the Transformer model. Our model has an embedding dimension of $64$, a hidden dimension of $128$, $2$ layers, and uses $2$ attention heads. The batch size is selected depend on the task: $16$ for Text and Retrieval, and $32$ for Listops. The random projection dimension of Macformer is set to $128$, and the $\epsilon$ for ppSBN is set to $1e\!\!-\!\!13$. Each training session involves $1000$ steps of initialization and $10000$ steps of optimization. We recorded the total training time, maximum memory usage, and final accuracy as result, and averaged these metrics across different random seeds.

Table \ref{macformerTable} presents the experimental results on the LRA benchmark. In the LRA Text task, RFA achieved the highest accuracy, but ${\rm Macformer}_{\rm exp}$, ${\rm Macformer}_{\rm inv}$, and ${\rm Macformer}_{\rm trigh}$ significantly reduced time consumption while maintaining accuracy close to RFA. ${\rm Macformer}_{\rm log}$ and ${\rm Macformer}_{\rm sqrt}$ also demonstrated relatively better time and accuracy compared to the base Transformer. In the LRA Listops task, ${\rm Macformer}_{\rm exp}$, ${\rm Macformer}_{\rm trigh}$, and ${\rm Macformer}_{\rm sqrt}$ achieved the best accuracy, with ${\rm Macformer}_{\rm trigh}$ also having the optimal time consumption, but ${\rm Macformer}_{\rm exp}$ did not improve training speed. In the LRA Retrieval task, ${\rm Macformer}_{\rm log}$ achieved the best training speed, albeit with a slight reduction in accuracy. ${\rm Macformer}_{\rm exp}$, ${\rm Macformer}_{\rm inv}$, ${\rm Macformer}_{\rm trigh}$, and ${\rm Macformer}_{\rm sqrt}$ also showed improved speed, with ${\rm Macformer}_{\rm trigh}$ maintaining a higher accuracy. Additionally, we found that the memory consumption of Macformer models with different kernel functions remains consistently the same. We believe this is because changing the kernel function only affects the selection of \(a_N\) in RMF, leading to a constant-level impact on the model.

The experimental results above validate our previous point: different similarity functions \( \mathcal{K}(\cdot) \) are suitable for different application scenarios. For each task, Macformer with different kernel functions exhibited significantly varying performance metrics. Therefore, one can choose Macformer with different kernel functions based on the application scenario and performance requirements.

\section{Conclusion}

We propose Macformer, a Transformer architecture based on RMF approximating dot-product kernels. Our analysis and experiments collectively demonstrate that Macformer exhibits superior efficiency when handling long sequence data, with performance backed by theoretical guarantees. Moreover, Macformer offers flexibility by allowing the selection of different kernelized attention based on application scenarios.

Future research directions may include exploring the effects of more dot-product kernels, determining how to select the optimal \( \mathcal{K}(\cdot) \) based on the application scenario of Macformer, and reducing the memory consumption of Macformer. 

\bibliography{aaai25}

\end{document}